\documentclass[letterpaper, 10 pt, conference]{ieeeconf}  

\IEEEoverridecommandlockouts                              

\usepackage{amsfonts,amssymb,amsmath}

\usepackage{graphicx}
\usepackage{glossaries}
\usepackage{tikz}
\usetikzlibrary{arrows.meta}
\usepackage{flushend}

\usepackage{algorithm}
\usepackage{algorithmic}

\usepackage{hyperref}
\hypersetup{
    colorlinks=true,
    linkcolor=blue,
    filecolor=magenta,      
    urlcolor=cyan,
}

\usepackage{amsthm}
\newtheorem{lemma}{Lemma}

\newacronym{bc}{BC}{Behavioural Cloning}
\newacronym{irl}{IRL}{Inverse Reinforcement Learning}
\newacronym{lfd}{LfD}{Learning from Demonstration}
\newacronym{em}{EM}{Expectation Maximization}
\newacronym{promp}{ProMP}{Probabilistic Movement Primitives}
\newacronym{dmp}{DMP}{Dynamic Movement Primitives}
\newacronym{seds}{SEDS}{Stable Estimator of Dynamical Systems}
\newacronym{gmr}{GMR}{Gaussian Mixture Regressor}
\newacronym{gpr}{GPR}{Gaussian Process Regressor}
\newacronym{lwr}{LWR}{Locally Weighted Regressor}
\newacronym{kmp}{KMP}{Kernelized Movement Primitives}
\newacronym{clf}{CLF}{Control Lyapunov Function}
\newacronym{wsaqf}{WSAQF}{Weighted Sum of Asymmetric Quadratic Function)}
\newacronym{nilc}{NILC}{Neurally Imprinted Lyapunov
Candidate}
\newacronym{clfdm}{CLF-DM}{Control Lyapunov Function-based Dynamic Movements}

\newacronym{gcl}{GCL}{Guided Cost Learning}

\newacronym{mle}{MLE}{Maximun Likelihood Estimation}
\newacronym{sde}{SDE}{Stochastic Differential Equation}
\newacronym{ode}{ODE}{Ordinary Differential Equation}

\newacronym{probs}{ProbS}{Probabilistic Segmentation}
\newacronym{crf}{CRF}{Conditional Random Fields}
\newacronym{ppca}{PPCA}{Probabilistic Principal Component Analysis}
\newacronym{gmcc}{GMCC}{Generalized Multiple Correlation Coeficcient}
\newacronym{hri}{HRI}{Human-Robot Interaction}
\newacronym{ip}{IP}{Interaction Primitives}
\newacronym{hmm}{HMM}{Hidden Markov Model}
\newacronym{cac}{CAC}{Canonical Correlation Coefficient}
\newacronym{rv}{$R_v$}{$R_v$ Coefficient}
\newacronym{dcor}{dCor}{Distance Correlation}
\newacronym{dtw}{DTW}{Dynamic Time Warping}
\newacronym{edr}{EDR}{Edit Distance With Real Penalty}
\newacronym{twed}{TWED}{Time Warp Edit Distance}
\newacronym{r2}{$R^2$}{Coefficient of Determination}
\newacronym{sqp}{SQP}{Successive Quadratic Programming}
\newacronym{rkhs}{RKHS}{Reproducing Kernel Hilbert Space}
\newacronym{icnn}{ICNN}{Input-Convex Neural Network}

\newacronym{pca}{PCA}{Principal Component Analysis}

\newacronym{maf}{MAF}{Masked Autoregressive Flow}
\newacronym{iaf}{IAF}{Inverse Autoregressive Flow}

\usepackage{amsmath,amsfonts,bm}

\def\1{\bm{1}}

\def\RR{\mathbb{R}}

\def\rmJ{{\mathbf{J}}}

\def\vtheta{{\bm{\theta}}}

\def\vg{{\bm{g}}}

\def\vp{{\bm{p}}}

\def\vy{{\bm{y}}}

\def\vz{{\bm{z}}}
\def\vB{{\bm{B}}}

\def\vlambda{{\boldsymbol{\lambda}}}

\def\vpsi{{\boldsymbol{\psi}}}
\def\vphi{{\boldsymbol{\phi}}}

\def\vtheta{{\boldsymbol{\theta}}}

\DeclareMathAlphabet{\mathsfit}{\encodingdefault}{\sfdefault}{m}{sl}
\SetMathAlphabet{\mathsfit}{bold}{\encodingdefault}{\sfdefault}{bx}{n}

\title{\LARGE \bf
ImitationFlow: Learning Deep Stable Stochastic Dynamic Systems by Normalizing Flows
}

\author{ 
Julen Urain$^1$, Michele Ginesi$^2$, Davide Tateo$^1$, and Jan Peters$^{1,3}$\thanks{
$^1$Intelligent Autonomous Systems, TU Darmstadt}
\thanks{$^2$Department of Computer Science, University of Verona}
\thanks{$^3$MPI for Intelligent Systems, Tuebingen}
\thanks{\tt \small \{urain,tateo,peters\}@ias.tu-darmstadt.de}
\thanks{\tt \small michele.ginesi@univr.it}
\thanks{This project has received funding from the European Union’s Horizon 2020 research and innovation programmes under grant agreement No. \#820807 (SHAREWORK), \#713010 (GOAL-Robots), and \#640554 (SKILLS4ROBOTS)}
}

\begin{document}

\maketitle
\thispagestyle{empty}
\pagestyle{empty}

\begin{abstract}
We introduce ImitationFlow, a novel Deep generative model that allows learning complex globally stable, stochastic, nonlinear dynamics. Our approach extends the Normalizing Flows framework to learn stable Stochastic Differential Equations. We prove the Lyapunov stability for a class of Stochastic Differential Equations and we propose a  learning algorithm to learn them from a set of demonstrated trajectories. Our model extends the set of stable dynamical systems that can be represented by state-of-the-art approaches, eliminates the Gaussian assumption on the demonstrations, and outperforms the previous algorithms in terms of representation accuracy. We show the effectiveness of our method with both standard datasets and a real robot experiment.
\end{abstract}

\section{INTRODUCTION}
To have fully autonomous robots in real-world scenarios, robots need to be able to perform a wide variety of complex tasks and skills.
However, programming a robot to perform complex motion is often a time-consuming task, which requires both expertise and domain knowledge. Imitation Learning tackles this problem and provides robotics non-expert users the capability to teach robots complex trajectories providing a few demonstrations \cite{schaal1997learning,atkeson1997locally}.

Given a set of demonstrations, the objective of Imitation learning is to find a generative model that produces trajectories similar to the demonstrated ones. A correct selection of the model will lead to a successful and safe imitation. These generative models are called Movement Primitives.
Movement primitives can be divided into three different categories: time-dependant, state-dependant, and time-state dependant movement primitives. 

In this work, we focus on state dependant motion primitives as they are robust to both spatial and temporal perturbation. This class of motion primitives is complementary to the time-dependent motion primitives approaches, such as the \gls{dmp}~\cite{schaal2006dynamic} and the \gls{promp}~\cite{paraschos2013probabilistic} approaches, that are particularly well suited when the movement presents a clear temporal (or phase) dependency, while robustness to perturbation is not a major concern.

\begin{figure}[t]
    \centering
    \includegraphics[width=.35\textwidth]{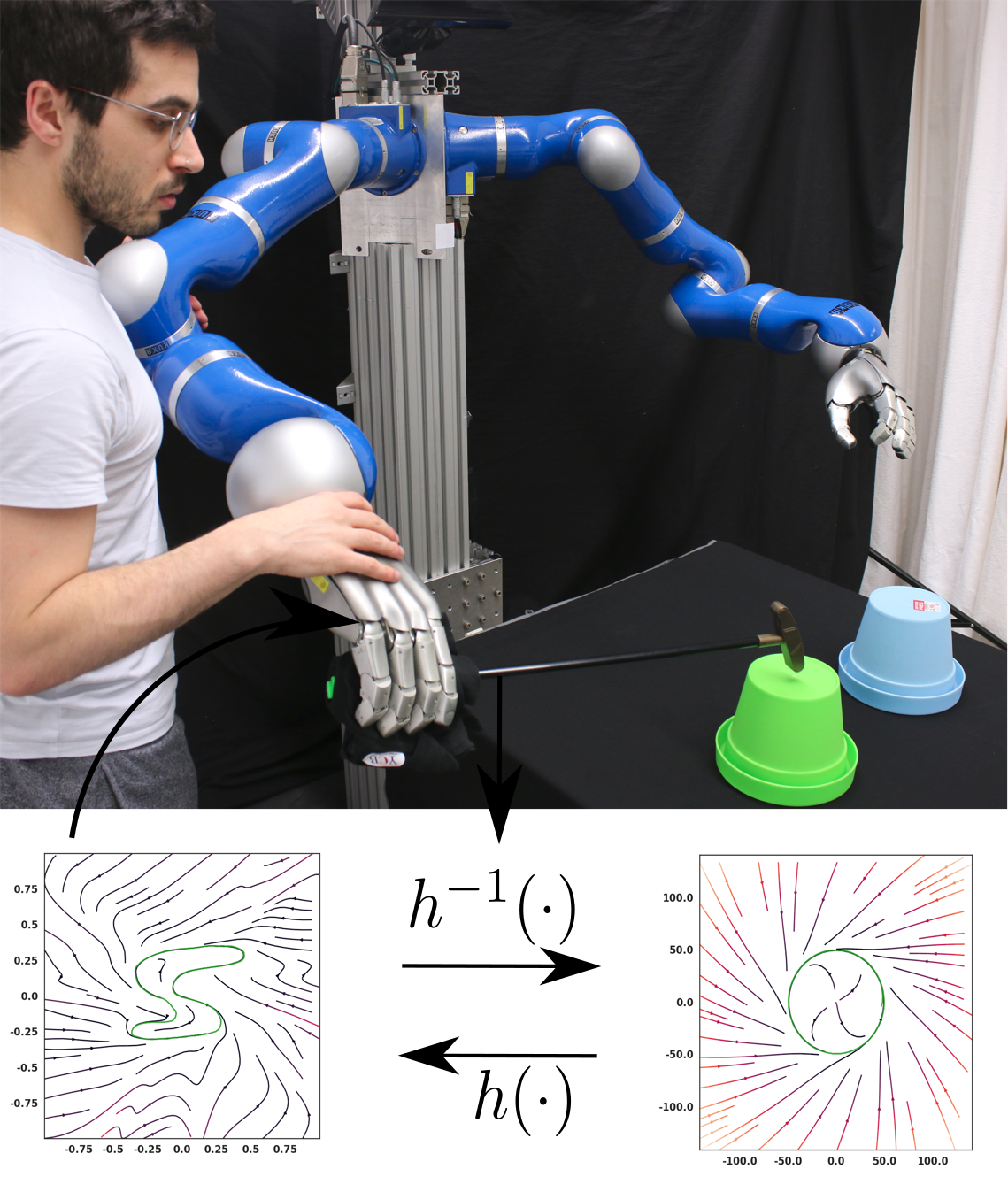}
    \caption{Above, robot is taught to perform a drums playing task. Below, \textit{ImitationFlows} receives the recorded trajectories and learns to morph the latent dynamics in the left to the dynamics in the right. The generated dynamics in the right are given back to the robot to perform the task.}
    \label{fig:real_task}
\end{figure}

One of the main issues of state-dependent movement primitives is to learn stable dynamics: While several models exist \cite{schaal2002scalable,hersch2008dynamical,huang2019kernelized}; most of them are not good ensuring stable dynamics out of the region of the demonstrated trajectories.
This issue limits the applications of these methods to real-world scenarios, as it can be difficult to ensure the correct execution of the learned motion.
This issue has been faced by the \gls{seds} algorithm~\cite{khansari2011learning}. In this model, the global asymptotically stability is ensured by a quadratic Lyapunov function. Due to the proposed Lyapunov function, the learned dynamics are restricted to continuously decreasing distance towards the attractor. In order to overcome the limitation of \gls{seds}, different approaches were proposed~\cite{lemme2013neurally, khansari2014learning,ravichandar2015learning,neumann2015learning, perrin2016fast,ravichandar2017learning,sindhwani2018learning,Kolter2019Learning}. However, except for \cite{umlauft2017learning}, all these models consider deterministic models. Also, all these works consider only point-to-point dynamics.

\textbf{Contributions} In this paper, we present \textit{ImitationFlow}, a novel approach for learning stable stochastic nonlinear dynamical systems.
Our methodology merges Deep generative models like Normalizing Flows \cite{rezende2015variational, dinh2016density, papamakarios2017masked} with stable dynamical systems modeling. 

Our approach not only is capable of representing a wider class of dynamical systems w.r.t. previous works in the field, but it can also represent arbitrary complex densities, removing the Gaussian noise assumption of the demonstrated trajectories.
The proposed model can describe both strike-based and periodic movements in a single framework, without changing the core learning algorithm.

We also formally prove that \textit{ImitationFlow} can learn globally asymptotically stable stochastic dynamics. This property ensures global convergence to the goal for point-to-point motions, ensuring that the learned movement can be applied from any starting point of the workspace. We have released a Pytorch Implementation at \href{https://github.com/TheCamusean/iflow}{https://github.com/TheCamusean/iflow}

\section{Background \& Notation}

\subsection{Problem Statement}\label{sec:prob_state}
Let $\mathcal{T} = \{ \tau_1, \tau_2, \dots , \tau_{n-1}, \tau_n \}$ be a set of $n$ expert demonstrated trajectories, where each trajectory $\tau_i = \{ \vy_{1}, \vy_{2} , \dots, \vy_{T_i} \}$ of length $l_i$ is a sequence of observations $\vy_i \in \RR^{d}$.
For clarity of presentation, we assume that each element of the trajectory is generated with a fixed sampling time $\Delta T$, but our derivations can be extended to the variable sampling time scenario.
 
Let $p(y_0)$ be the distribution of the trajectory starting point. We assume that each trajectory element is generated by the following \gls{sde}:
\begin{align}\label{eq:sde}
    d\vz(t) = f(\vz(t))dt + \vg(\vz(t) ) d \vB(t),
\end{align}
where $\vz \in \RR^d$ is the state, $ f : \RR^d \to \RR^d $ is a continuous function, $ \vg : \RR^d \to \RR^{d\times d}  $ is a continuous matrix function, and $ \vB : \RR \to \RR^d $ is a $d-$dimensional \emph{Brownian motion} (also called \emph{Wiener process}).\\
We remark that \eqref{eq:sde} is written in autonomous form. The general case, in which $f$ and $\vg$ are functions of both the state and time can be easily recovered by defining $ \tilde{ \vz } = [\vz ^{\intercal}, t]^{\intercal} $, $ \tilde{f}(\tilde{\vz}) = [f(\vz(t))^\intercal, 1]^{\intercal} $.

Our problem is to maximize the likelihood of the unknown model parameters $\vpsi = \{ \vtheta, \vphi \}$ w.r.t. the set of observed trajectories $\mathcal{T}$.
We frame our learning problem as the following optimization problem
\begin{align}
    \vpsi^{*} = & \arg \max_{\vpsi} \mathcal{L}_{\vpsi}\left(\mathcal{T}\right) \nonumber\\
    s.t. &\lim_{t\to\infty} \mathbb{E}\left[\vy_t\right] \in \Omega,\, \forall y_0 \in \RR^{d}
    \label{eq:likelihood_optim}
\end{align}
where $\Omega\subset \RR^{d}$ and $\mathcal{L}_{\vpsi}$ is defined as
\begin{equation*}
     \mathcal{L}_{\vpsi}\left(\mathcal{T}\right) = \prod_{i=1}^{n} p(\tau_{i};\vpsi) = \prod_{i=1}^{n}p(y_0)\prod_{j=0}^{l_i} p(y_{j+1}|y_j).
\end{equation*}

\subsection{Stability for Stochastic Dynamical systems}\label{sec:stoLyapuSta}

Lyapunov stability has various definitions in the case of stochastic dynamical systems.
One option is to study the stability in probability, in which the stability is studied replacing $\dot{V}$ by its expected value~\cite{mao2007stochastic}.

By It\^{o}'s formula, the time derivative of the Lyapunov function,  $V(\vy, t)$, is expressed by the following differential equation
\begin{align}
    dV(\vy,t) = LV(\vy,t) dt + V_{\vy}(\vy,t) \vg(\vy , t) d \vB(t) \nonumber\\
    LV(\vy,t) = V_t + V_{\vy} f(\vy,t) + \frac{1}{2} \textrm{Tr}(\vg^{\intercal}(\vy , t)) V_{\vy\vy}\vg(\vy , t),
    \label{eq:sto_derivative}
\end{align}
where $LV(\vy,t)$, is known as the diffusion equation. 
The expected value of $\dot{V}$ is
\begin{align}
    \mathbb{E}\left[\dot{V}(\vy,t)\right] = LV(\vy,t).
\end{align}
Assume there exist a $V(\cdot,\cdot) : \RR^{d} \times \RR \xrightarrow{} \RR$, $ (\vy, t)\mapsto V(\vy, t) $, and strictly increasing functions $\mu_1, \mu_2, \mu_3$ such that
\begin{subequations}
\begin{align}
    \mu_1(|\vy|) &\leq V(\vy,t) \leq \mu_2(|\vy|), \label{eq:stLyapu1} \\
    LV(\vy,t) &\leq - \mu_3(|\vy|) \, \forall \vy \in \RR^{d}  \label{eq:stLyapu2}.
\end{align}
\end{subequations}
Then, the trivial solution for our \gls{sde} is stochastically asymptotically stable~\cite{mao2007stochastic}.

\subsection{Normalizing Flows}
Normalizing Flows provide a method for expressive density approximation.  ~\cite{rezende2015variational}. Requiring only a base distribution, usually uniform or normal distribution, and a set of bijective transformations, Normalizing Flows allows explicit density estimation and, in most cases, to sample from these complex distributions ~\cite{papamakarios2017masked,dinh2016density}. 

Given a latent variable $\vz \in \RR^{d}$, sampled from a certain distribution $z \sim p_z(\vz)$ and a diffeomorphic transformation $h:\; \RR^{d} \to  \RR^{d}$, such that $\vy = h(\vz)$, we can compute the distribution of $\vy$, $p_y(\vy)$, in terms of $\vz$, by the change of variable rule.
\begin{align}
    p_y(\vy) = p_z(\vz) \left|\det \frac{\partial \vz}{\partial \vy} \right| = p_z(h^{-1}(\vy)) \left|\det \frac{\partial h^{-1}(\vy)}{\partial \vy} \right|.
\end{align}

\section{PROPOSED METHOD}\label{ImitationFlow}

\begin{figure}[t]
  \centering
  \resizebox{0.8\columnwidth}{!}{
    \begin{tikzpicture}[scale=0.75]
        \node[draw, circle] (z0) at (0,0) {$ Z_0 $};
        \node[draw, circle] (z1) at (2,0) {$ Z_1 $};
        \node[draw, circle] (zn) at (5,0) {$ Z_n $};
        \node[fill = gray, circle, opacity=0.3] (y0) at (0,-2) {$ Y_0 $};
        \node[fill = gray, circle, opacity=0.3] (y1) at (2,-2) {$ Y_1 $};
        \node[fill = gray, circle, opacity=0.3] (yn) at (5,-2) {$ Y_n $};
        \node[draw, circle] (y0) at (0,-2) {$ Y_0 $};
        \node[draw, circle] (y1) at (2,-2) {$ Y_1 $};
        \node[draw, circle] (yn) at (5,-2) {$ Y_n $};
        \draw[-{Latex}] (z0) -- (z1);
        \draw[-{Latex}, dashed] (z1) -- (zn);
        \draw[-{Latex}] (z0) -- (y0);
        \draw[-{Latex}] (z1) -- (y1);
        \draw[-{Latex}] (zn) -- (yn);
        
        \draw (-0.75,0.75) -- (5.75,0.75) -- (5.75, -2.75) -- (-0.75,-2.75) -- cycle;
        
        \node[draw, circle] (phi) at (-1.5, 1) {$ \vphi $};
        \node[draw, circle] (theta) at (-1.5, -1) {$ \vtheta $};
        
        \draw[-{Latex}] (phi) to (z0);
        \draw[-{Latex}] (phi) to [out = 0, in = 150] (z1);
        \draw[-{Latex}] (phi) to [out = 0, in = 160] (zn);
        
        \draw[-{Latex}] (theta) to (y0);
        \draw[-{Latex}] (theta) to [out = 0, in = 150] (y1);
        \draw[-{Latex}] (theta) to [out = 0, in = 160] (yn);
    \end{tikzpicture}
  }
  \caption{Imitation Flow architecture as a graphical model}
  \label{im:ImitationFlow}
\end{figure}
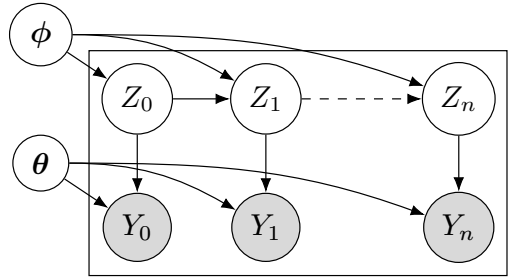
\textit{ImitationFlow} model extends the concept of Structured Inference Networks for Nonlinear State Space Models~\cite{krishnan2017structured} to Normalizing Flows. Considering Normalizing Flows as emission function allows exact inference at the cost of imposing a deterministic emission.

The proposed model's architecture is presented in Fig.~\ref{im:ImitationFlow}, and the dynamics are modelled using \eqref{eq:ImitFlow}.
Our model is composed of two main components. In the latent space $\mathcal{Z}$, the transition model follows some stable stochastic dynamics parameterized by $\vphi$. 
Then, we use, as emission function, a bijective, continuous and differentiable transformation $h_{\vtheta}\,: \RR^{d}\xrightarrow{}\RR^{d}$ transforming the data from the latent space $\mathcal{Z}$ to the observation space $\mathcal{Y}$. 
\begin{align}
    & d\vz(t) = f_{\vphi}(\vz)dt + \vg_{\vphi}(\vz) d \vB(t) \nonumber\\
    & \vy = h_{\vtheta}(\vz),
\label{eq:ImitFlow}
\end{align}
where $\vtheta$ and $\vphi$ are the learnable parameters of the model.
Given that the Jacobian of $h_{\vtheta}$, $J_{\vtheta} = \frac{d \vy}{d \vz}$, is easy to compute, we can reframe \eqref{eq:ImitFlow} to compute the stochastic dynamic model for $\vy(t)$
\begin{align}
    d \vy(t) = J_{\vtheta}(\vy)f_{\vphi}(h_{\vtheta}^{-1}(\vy))dt + J_{\vtheta}(\vy)\vg_{\vphi}(h_{\vtheta}^{-1}(\vy))d \vB(t)
    \label{eq:observ_dyn}
\end{align}

\subsection{Learning Algorithm}\label{sec:LearningAlgo}




In this section, we describe how we solve the optimization problem described in \eqref{eq:likelihood_optim}.
In order to estimate correctly the likelihood of the trajectories, we should estimate the probability density of initial states $p(\vy_0)$. However, if a few trajectories are available, this estimation can be problematic.
We leverage on the fact that the system represents stable dynamics: this means that we will have a stationary distribution in the limit
\begin{align*}
    \lim_{t \xrightarrow{} \infty} p(\vy_t) = p(\vy_\infty).
\end{align*}
We exploit this property by assuming that the distribution of the last point of the trajectory is the stationary one.
Differently from most of others probabilistic estimation algorithms, where the structural density is considered with forward conditioning, $p(\vy_{j+1}| \vy_{j})$, in our approach we consider a backward conditioning with $p(\vy_{j}| \vy_{j+1})$. It is straightforward to prove that this view is equivalent under Bayes's rule.

Given the model proposed in \eqref{eq:ImitFlow}, we can rewrite the probability distributions $p(\vy)$ in terms of $p(\vz)$. By applying the change of variable rule we obtain
\begin{align}
    p(\vy_n) & = p(\vz_n) \left| \det \frac{\partial \vz_n}{\partial \vy_n} \right| = p(f^{-1}(\vy_n)) \left| \det J^{-1}(\vy_n)\right|.
    \label{eq:density_1}
\end{align}
Due to the fact that $\vy_{i+1}$ and $\vz_{i+1}$ are the same event, it follows that $p(\vy_i|\vy_{i+1}) = p(\vy_i|\vz_{i+1})$. Applying again the change of variable rule we obtain
\begin{align}
    p(\vy_i|\vy_{i+1}) & = p(\vz_i|\vz_{i+1}) \left| \det \frac{\partial \vz_i}{\partial \vy_i} \right| \nonumber\\
    & = p(f^{-1}(\vy_i)|f^{-1}(\vy_{i+1})) \left| \det J^{-1}(\vy_i) \right|. \label{eq:density_2}
\end{align}

To optimize \eqref{eq:density_2} we first select a simple stochastic dynamical system for the latent dynamics. We choose a parameterization such that the asymptotic behavior of the system e.g., stable equilibrium point or limit cycle, is enforced. Then we select a suitable class of normalizing flows as emission function $h$.
In this work, we will use the Euler-Maruyama~\cite{platen2010numerical} integration scheme to integrate the \gls{sde} that describes the dynamics of the latent space.

The resulting  algorithm is summarized in Alg.~\ref{alg:learning}. On each iteration,  a random trajectory $\tau_\vy$ and a sampling time $\Delta T$ are selected. By the inverse of the bijective transformation $h^{-1}(\cdot)$ the latent trajectory $\tau_\vz$ is computed. The determinant of the inverse jacobian $|J|^{-1}$ is also computed as we require it for computing the probability in (\ref{eq:density_1}) and (\ref{eq:density_2}). The latent trajectory $\tau_\vz$ is split to be able to compute the conditional probability $p(\vz_i|\vz_{i+\Delta T})$ and the probability in the end $p(\vz_n)$.

\begin{algorithm}
\caption{ImitationFlow Learning}
\begin{algorithmic}
    \STATE $\textrm{Input: } \mathcal{T} \textrm{ trajectories }$
\STATE $\textrm{Parameters: } \vphi \textrm{ dynamics, } \vtheta \textrm{ NormalizingFlow}$
\WHILE{$\textrm{not converged}$}
\STATE $\tau_\vy \leftarrow \{ \mathcal{T} \}$
\STATE $\Delta T \leftarrow \{\textrm{Get a sampling time}\}$
\STATE $\tau_\vz , |\rmJ_{\tau_\vz}|^{-1} \leftarrow h^{-1}_{\vtheta}(\tau_\vy)$
\STATE $\vz_{(0:T-\Delta T)} , \vz_{(\Delta T:T)} , \vz_n \leftarrow \textrm{SplitTime}(\tau_\vz, \Delta T)$
\STATE $\vp(\cdot|\vz_{i+\Delta T};\vphi) , p_n(\cdot;\vphi) \leftarrow \textrm{GetDensFunc}(\vz_{(\Delta T:T)} , \vz_n)$
\STATE $\mathcal{L} = p_n(\vz_n;\vphi)|J_n|^{-1} \prod p(\vz_{i} |\vz_{i+\Delta T};\vphi)|J_i|^{-1}$
\STATE $\Delta \vtheta, \Delta \vphi \propto - \nabla\vtheta \mathcal{L}, - \nabla\vphi \mathcal{L} $
\ENDWHILE
\end{algorithmic}
\label{alg:learning}
\end{algorithm}

\subsection{Latent Stochastic Linear Dynamics} \label{sec:lat_linear}
A simple stochastic stable dynamic system is the stochastic Linear dynamics
\begin{align}
    d\vz(t) = A_{\vphi} \vz(t) dt + K_{\vphi} d \vB(t) 
\end{align}
where $A_{\vphi}$ and $K_{\vphi}$ are the learnable parameters of the latent dynamics. We impose stability by constraining the eigenvalues real part to be smaller than zero, $\mathcal{R}(\vlambda_A)<0$.

Besides, we require to compute the stationary distribution of $\vz$. To simplify the problem, we consider the stationary distribution of the discretized system. Let $F=A\Delta T+I$ and $\Sigma = KK^T\Delta T$, then, we know that
\begin{align*}
    \lim_{t \xrightarrow{}\infty} p(\vz_t) = \mathcal{N}(\boldsymbol{0},\Sigma_\infty), & &
    \Sigma_\infty = \sum_{i=0}^{\infty}F^{i} \Sigma F^{i^{\intercal}},
\end{align*}
where $\Sigma_{\infty}$ can be computed in closed form in the vectorized form as
\begin{equation}
    \textrm{vec}(\Sigma_{\infty}) = \textrm{vec}(\Sigma) (I_{n^{2}} - F \otimes F).
    \label{eq:sigma_inf}
\end{equation}

To learn the model it is also required to compute the backward conditional probability $p(\vz_i| \vz_{i-1})$. Linear stochastic dynamics are invertible. Therefore, we can compute the conditional backward probability, which is normally distributed
\begin{equation*}
    p(\vz_i| \vz_{i+1}) = \mathcal{N}(\vz_i| F^{-1}\vz_{i+1}, F^{-1}\Sigma F^{-\intercal}).
\end{equation*}

\subsection{Stability analysis for latent stable dynamics}
In the following, we prove the asymptotic stability in probability of the learned system, under the assumption of a stable latent dynamic.
\begin{lemma}
For any diffeomorphic transformation $\vy(t) = h (\vz(t))$, if the dynamics of $\vz(t)$ are stochastically asymptotically stable, then the dynamics of $\vy(t)$ are also stochastically asymptotically stable.
\end{lemma}

\begin{proof}
We follow a derivation similar to the one proposed in \cite{neumann2015learning} for Lyapunov stability analysis in \gls{ode}.
Let $U(\cdot)$ be a Lyapunov function for the latent dynamics.
We define the following Lyapunov candidate for the observed dynamics:
\begin{equation*}
    V(\vy) = U(h^{-1}(\vy)).
\end{equation*}
From this definition it follows that
\begin{equation}
    U(\vz) = V(h(\vz)).
    \label{eq:cond1}
\end{equation}
From the equality in \eqref{eq:cond1} and by the fact that U is a valid Lyapunov candidate we get that the condition in \eqref{eq:stLyapu1} is also satisfied.

To satisfy the condition in \eqref{eq:stLyapu2}, we compute the diffusion \eqref{eq:sto_derivative} of the Lyapunov function $LV$. Considering \eqref{eq:sto_derivative} and \eqref{eq:observ_dyn}
\begin{align}
       LV(\vy,t) = & V_t + V_{\vy} J(\vy)f(h^{-1}(\vy)) + \nonumber \\ 
       & \frac{1}{2} \textrm{Tr}((J(\vy)\vg(h^{-1}(\vy)))^{\intercal} V_{\vy\vy} (J(\vy)\vg(h^{-1}(\vy)))).
       \label{eq:lv}
\end{align}
Moreover, we can rewrite $V_t$, $V_{\vy}$ and $V_{\vy\vy}$ in terms of $U_{\vz}$ and $U_{\vz\vz}$.
\begin{align}
     V_t(\vy) = \frac{\partial}{\partial t} V(\vy) = 0,
     \label{eq:vt}
\end{align}
as we don't have explicit time dependency on $V(\vy)$.
\begin{align}
    V_{\vy}(\vy) = \frac{\partial}{\partial \vy} V(\vy) = \frac{\partial}{\partial \vz} U(\vz) \frac{\partial \vz}{\partial \vy} = U_{\vz}(\vz)J^{-1}(\vy),
    \label{eq:vy}
\end{align}
where $U_{\vz}(\cdot) \,:\, \RR^{d} \xrightarrow{} \RR^{1 \times d}$. Finally, we can rewrite $V_{\vy\vy}$
\begin{align}
    V_{\vy\vy}(\vy) &= \frac{\partial^{2}}{\partial \vy^{2}}V(y) = \frac{\partial}{\partial \vy^{\intercal}} \Big( \frac{\partial}{\partial \vy} V(\vy) \Big) = \nonumber \\
    &= \frac{\partial}{\partial \vy^{\intercal}} \Big(U_{\vz}(\vz)J^{-1}(\vy) \Big) 
        \label{eq:vyy} \nonumber\\
    &= \frac{\partial \vz^{\intercal}}{\partial \vy^{\intercal}}\frac{\partial}{\partial \vz^{\intercal}}\Big(U_{\vz}(\vz)J^{-1}(\vy) \Big) \nonumber\\
    & = J^{-\intercal}(\vy) U_{\vz\vz}(\vz) J^{-1}(\vy).
\end{align}
Introducing \eqref{eq:vt}, \eqref{eq:vy} and \eqref{eq:vyy} in the diffusion equation~\eqref{eq:lv}
\begin{align}
    LV(\vy) = &  U_{\vz} f(\vz) +\frac{1}{2} \textrm{Tr}(\vg^{\intercal}(\vz) U_{\vz\vz} \vg(\vz)) = LU(\vz).
    \label{eq:lu}
\end{align}
By hypothesis $LU(\vz)$ satisfies the condition in \eqref{eq:stLyapu2}, therefore the condition is also satisfied by $LV(\vy)$.
\end{proof}

\subsection{Latent Stochastic Limit Cycles}\label{sec:lat_limit}
In two dimensional space, attractive limit cycles can be represented by Linear Dynamics in polar coordinates $\{\rho, \psi \}$, and then apply the change of Variable rule to move to cartesian coordinates. Given $\rho \in \RR$, the radius and $\psi \in \{- \pi, \pi \}$, the angle as the latent state, the dynamics can be represented by stochastic linear dynamics
\begin{align}
    d \rho(t) &= a_{\vphi} (\rho(t) - \rho^{*}_{\vphi})dt  + \sigma_{1\vphi} dB(t) &\\
    d \psi(t) &= b_{\vphi}dt  + \sigma_{2\vphi} dB(t),
\end{align}
where $\{a_{\vphi}, b_{\vphi}, \rho^{*}_{\vphi} \}$ are learnable linear dynamics parameters and $\{\sigma_{1\vphi}, \sigma_{2\vphi}\}$ are learnable standard deviation parameters.
The state is transformed from polar coordinates to cartesian coordinates and back with a diffeomorphic transformation, so it allows us to apply the change of variable rule
\begin{align*}
    x = & \rho \cos (\psi) & & & \rho = & \sqrt{x^{2} + y^{2}} \\
    y = & \rho \sin (\psi) & & & \psi = & \arctan \left( \frac{y}{x} \right).
\end{align*}
Similarly to section \ref{sec:lat_linear}, in order to compute the loss, we need the stationary distribution $p(x_{\infty},y_{\infty})$. Applying the change of Variable rule, we obtain
\begin{equation*}
    p(x_{\infty},y_{\infty}) = p(\rho_{\infty},\psi_{\infty})\left|\frac{\partial (\rho_{\infty},\psi_{\infty})}{\partial (x_{\infty},y_{\infty})}\right|,
\end{equation*}
where the probability is split between the determinant of the Jacobian for the polar coordinates and the probability of the polar coordinates in $\infty$. The determinant of the Jacobian is
\begin{align}
    \det J = \det \begin{bmatrix}
\frac{x}{\sqrt{x^{2} + y^{2}}} & \frac{-y}{x^{2} + y^{2}}\\
\frac{y}{\sqrt{x^{2} + y^{2}}} & \frac{x}{x^{2} + y^{2}}
\end{bmatrix} = (x^{2} + y^{2})^{-\frac{1}{2}}.
\label{eq:jacobian_limit}
\end{align}
The density of the stationary distribution is the product of two separate distribution families. While, $\rho$ evolves with a linear dynamic towards $\rho^{*}$, leading to a normal distribution, the stationary distribution for $\psi$ depends on the initial density, $p(\psi_0)$. Unfortunately, in full generality, the density at the limit is not uniquely given, as it depends on the initial density distribution. For instance, if we assume that the initial phase distribution is uniform, the stationary density will be
\begin{equation*}
    p(\rho_{\infty},\psi_{\infty}) = p(\rho_{\infty})p(\psi_{\infty}) = \mathcal{N}(\rho^{*},\Sigma_{\infty})\mathcal{U}(-\pi,\pi).
\end{equation*}
where $\Sigma_{\infty}$ is computed the same way we did in \eqref{eq:sigma_inf}. We compute the conditional probability as in the previous section.

For higher dimensions, we propose to maintain the limit cycle in a two-dimensional plane and add further linear attractor dynamics towards zero for the extra coordinates.

\section{Experimental Evaluation}\label{exp}
In this section, we evaluate the learning quality of our model in different datasets. We compare \textit{ImitationFlow} with \gls{seds} and \gls{clfdm} in the LASA dataset \cite{khansari2011learning}. Then, we show that our model is capable of representing limit cycle dynamics, for both generation and discrimination of trajectories. Finally, we show that our algorithm scales to real-world problems by learning drums playing motion in a real robot.

\begin{figure*}
\centering
\begin{tabular}{ c  c c c }
    \includegraphics[width=.22\textwidth]{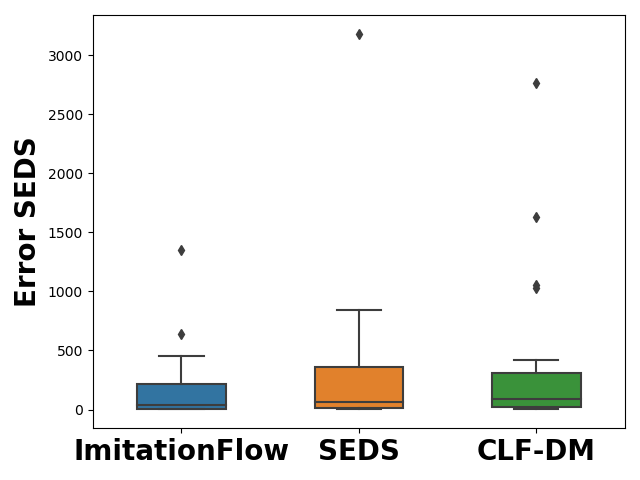} &
     \includegraphics[width=.22\textwidth]{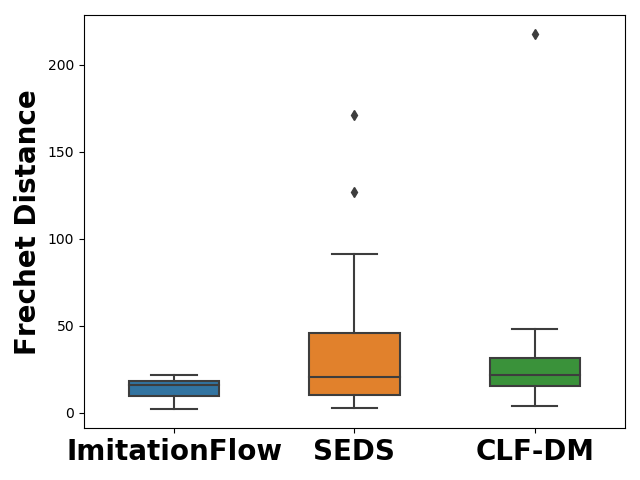} &
     \includegraphics[width=.22\textwidth]{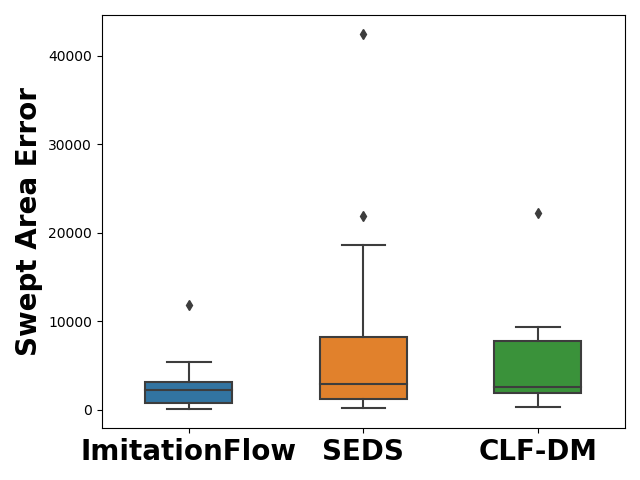} &
      \includegraphics[width=.22\textwidth]{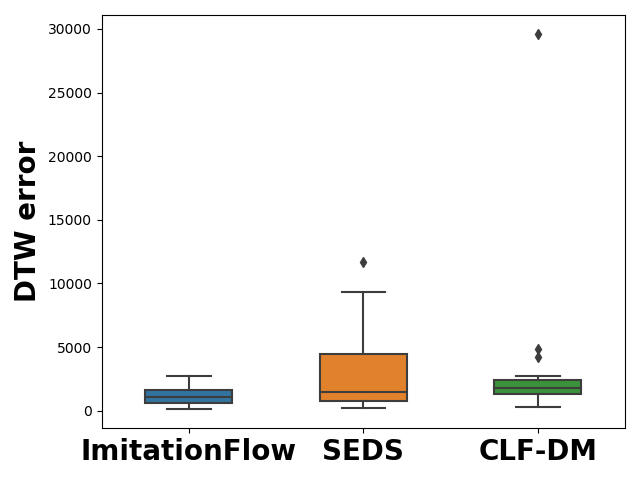}\\
      {\scriptsize (a) Error SEDS} &
      {\scriptsize (b) Frechet Distance} &
      {\scriptsize (c) Swept Area} &
      {\scriptsize (d) DTW distance}
  \end{tabular}
  \caption{Comparison of the performances of our method w.r.t. different state of the art algorithms. Boxplots are obtained using the LASA dataset.}
  \label{fig:metrics}
\end{figure*}

\subsection{Point-to-Point trajectories}\label{sec:lasa}

LASA dataset is a set of two-dimensional point-to-point trajectories with different shapes. This dataset is widely used for testing stable dynamics models. We learn the dynamics and we compare our results with baseline methods such as \gls{seds} and \gls{clfdm} in 4 metrics: \gls{seds} error \cite{khansari2011learning}, Swept Area Error \cite{khansari2014learning}, Frechet distance and DTW error \cite{Jekel2019}. 

As we deal with a small set of demonstrations, the mean velocity of the demonstrated trajectories is set as the initial velocity of the latent linear dynamics. With random initialization parameters, while the model can still match the trajectory closely, it struggles to learn the proper magnitudes of velocities.
For these experiments, we model the emission function as a concatenation of 10 coupling layers \cite{dinh2016density} and 10 orthogonal transformations \cite{kingma2018glow}.

To compute the metrics, we generate the expected trajectory from the same starting point of demonstrations and we compute the similarity w.r.t the demonstrations. We compare the obtained similarity measurements with the ones obtained for \gls{seds} and \gls{clfdm}.

Fig.~\ref{fig:metrics} shows that \textit{ImitationFlow} outperforms the other algorithms in all the metrics, providing more accurate results. Moreover, we can see that our model is more robust than \gls{seds} and \gls{clfdm}: it is clear that \textit{ImitationFlow} is the less variant and produces fewer outliers. 

The generated trajectories are shown in Fig.~\ref{fig:generated}. \textit{ImitationFlow} can learn stochastic dynamic systems that follow the trajectories provided by the user while maintaining stability. The vector field in Fig.~\ref{fig:generated} has been computed with the expected value and we can observe the variance in our model in the light blue trajectories. 

\begin{figure}[t]
    \centering
        \includegraphics[width=.22\textwidth]{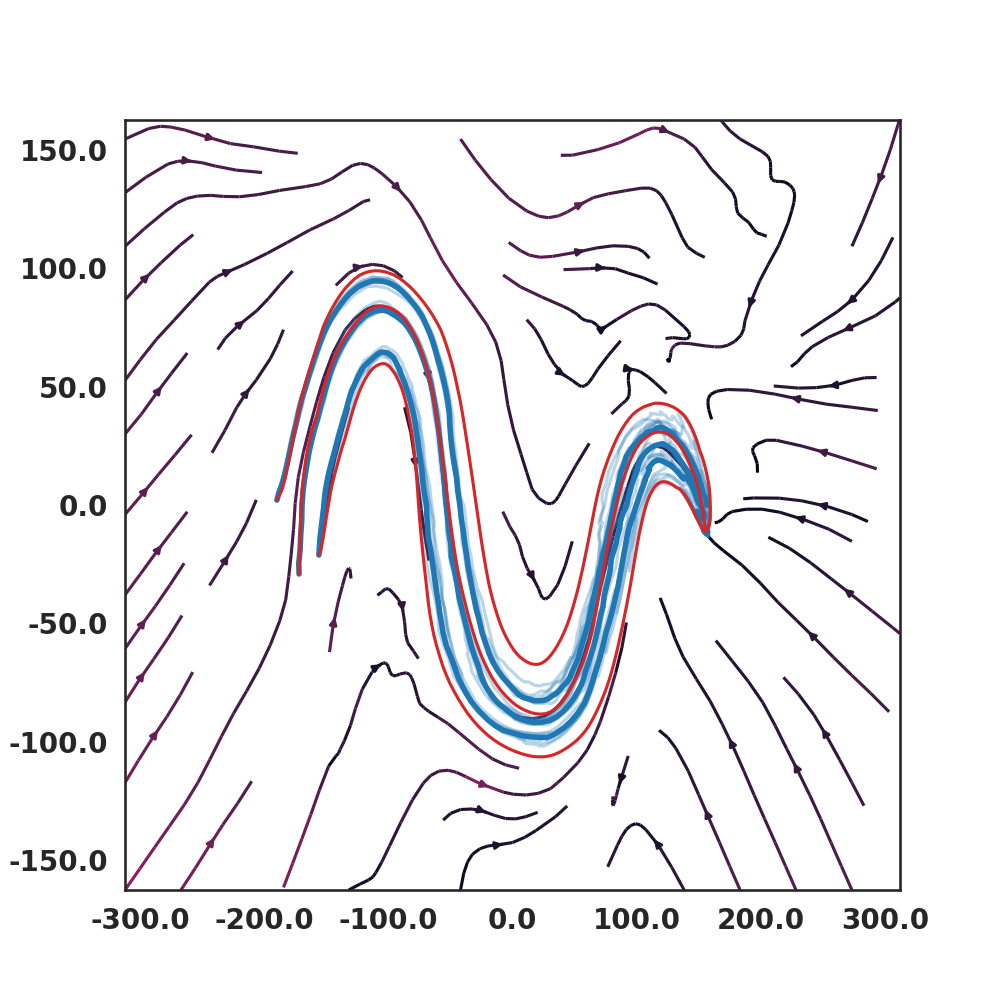}
    \hfill
        \includegraphics[width=.22\textwidth]{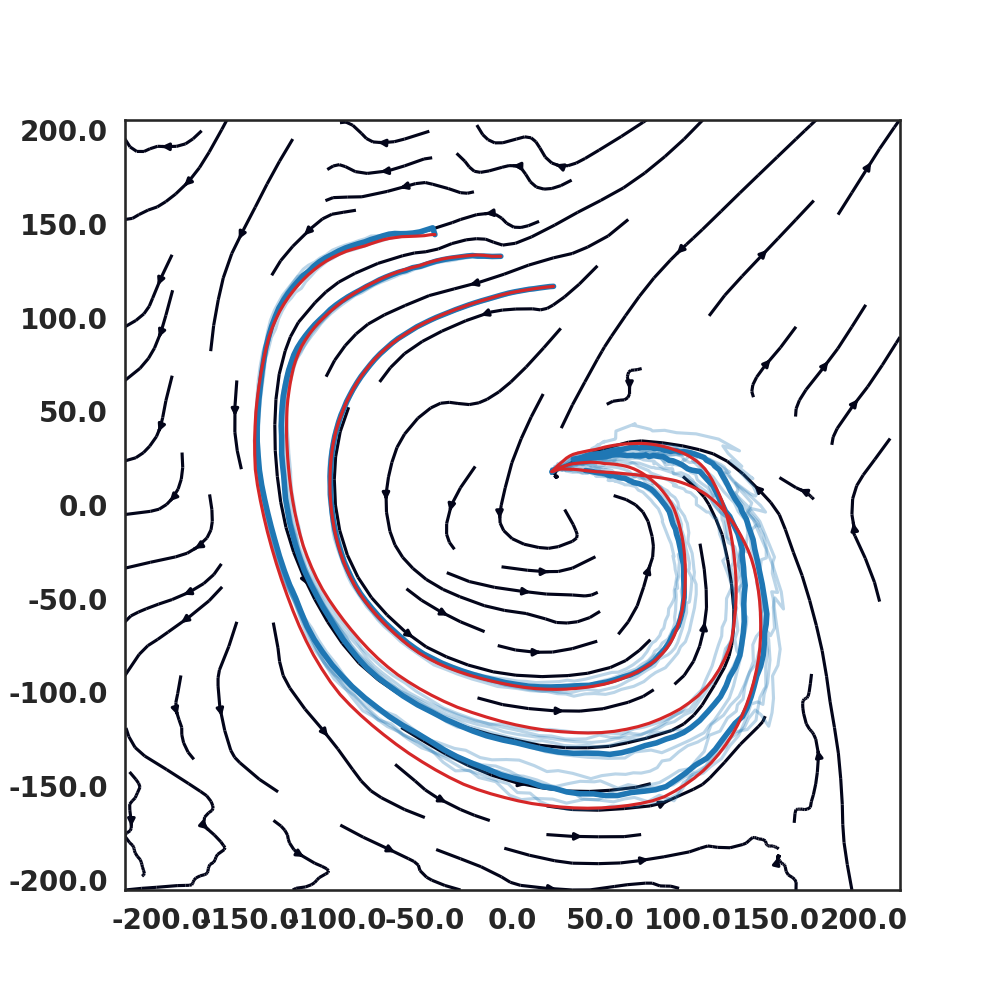}
    \caption{Examples of learned dynamics in LASA dataset. In dark blue, the expected trajectory, in light blue, noisy trajectories, in red, given trajectories}
    \label{fig:generated}
\end{figure}

\subsection{Limit Cycle behaviours}\label{sec:limit_toy}

We can model a complex limit cycle in the observation space by switching the latent linear dynamic model with a simple limit cycle behavior.
To prove this, we have recorded two-dimensional handwritten data, representing different letters. The recorded demonstrations maintain similar shape, orientation, and velocity in the drawings. 

Similarly to the previous section, proper initialization of the latent dynamics increases the performance of our model. We apply \gls{pca} over the trajectories and we extract the main frequency of the cycle through Fourier transform. This frequency is set as the initial angular velocity. 
For the nonlinear emission, we considered a concatenation of ten Coupling layers with ten orthogonal transformations. The dynamics generated by \textit{ImitationFlows} are shown in Fig.~\ref{fig:iros}. The model can perfectly represent the complex demonstrated dynamics.

\begin{figure*}[t]
    \centering
    \includegraphics[width=.24\textwidth]{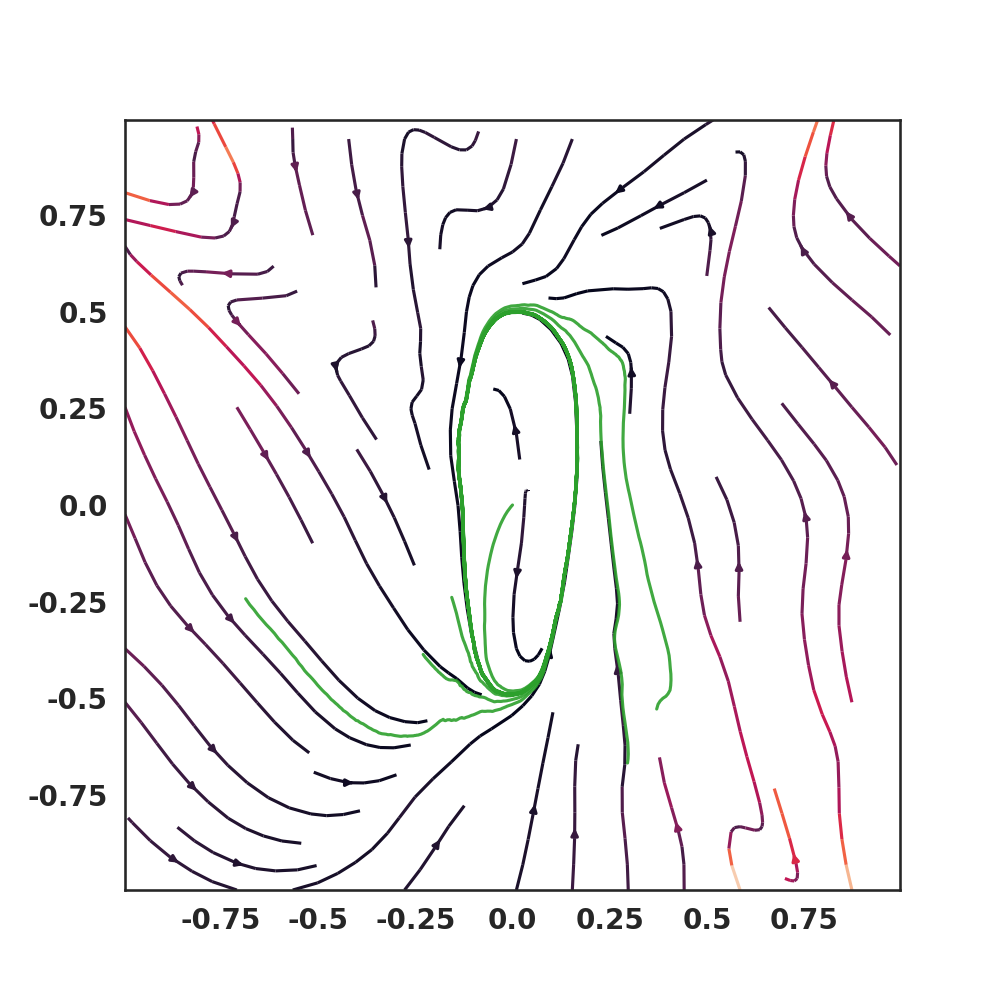}
    \hfill
        \includegraphics[width=.24\textwidth]{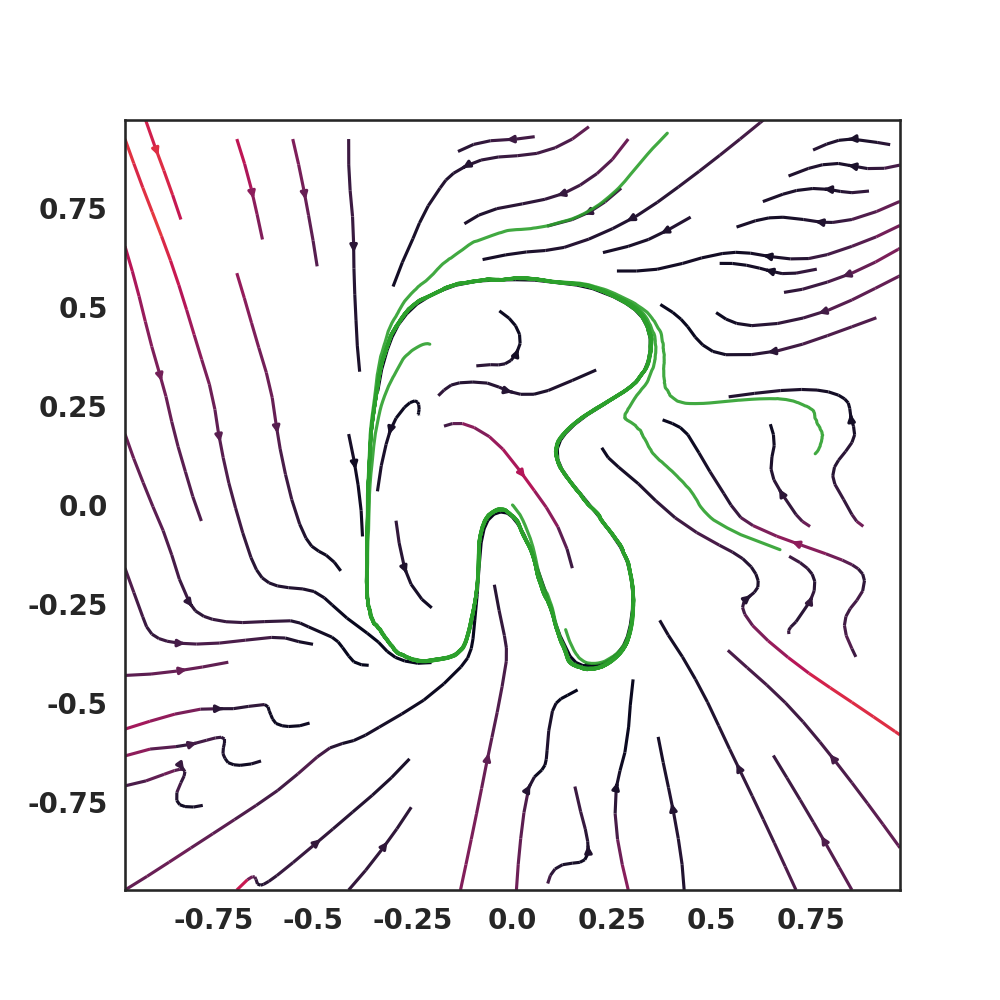}
    \hfill
        \includegraphics[width=.24\textwidth]{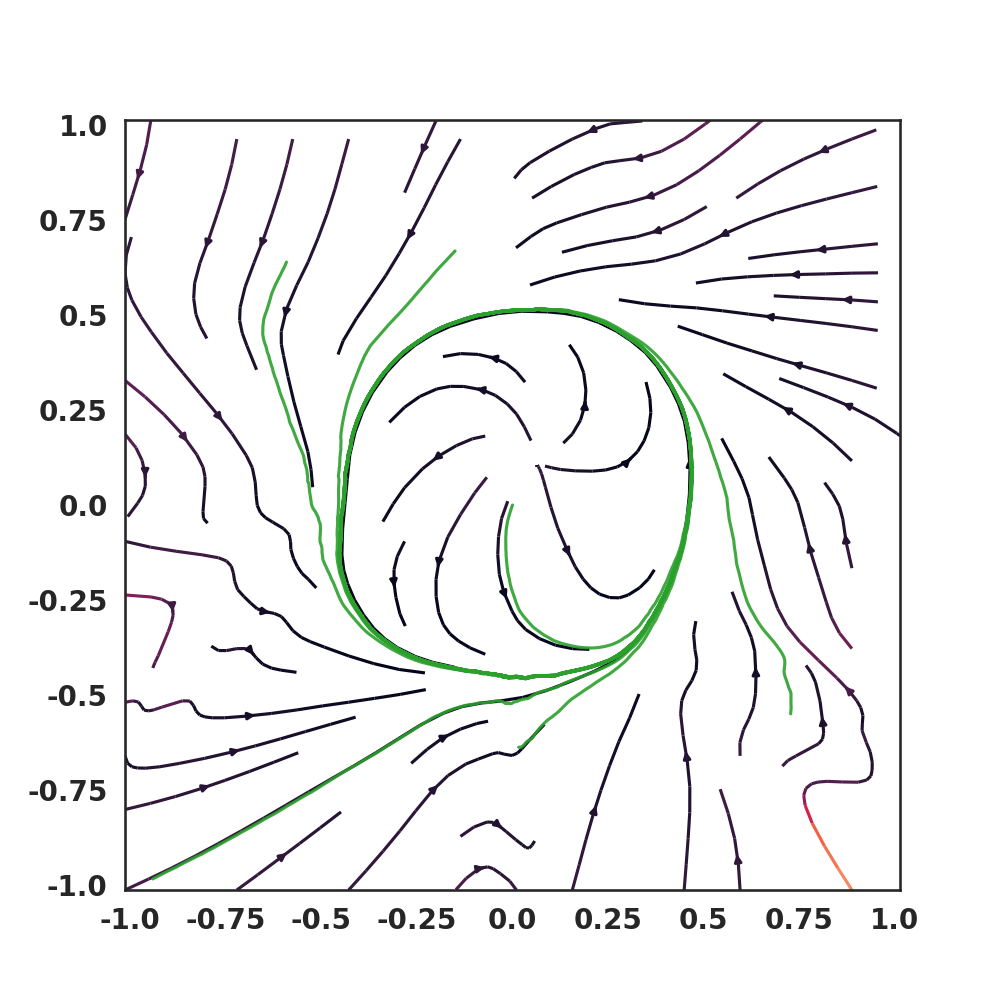}
    \hfill
        \includegraphics[width=.24\textwidth]{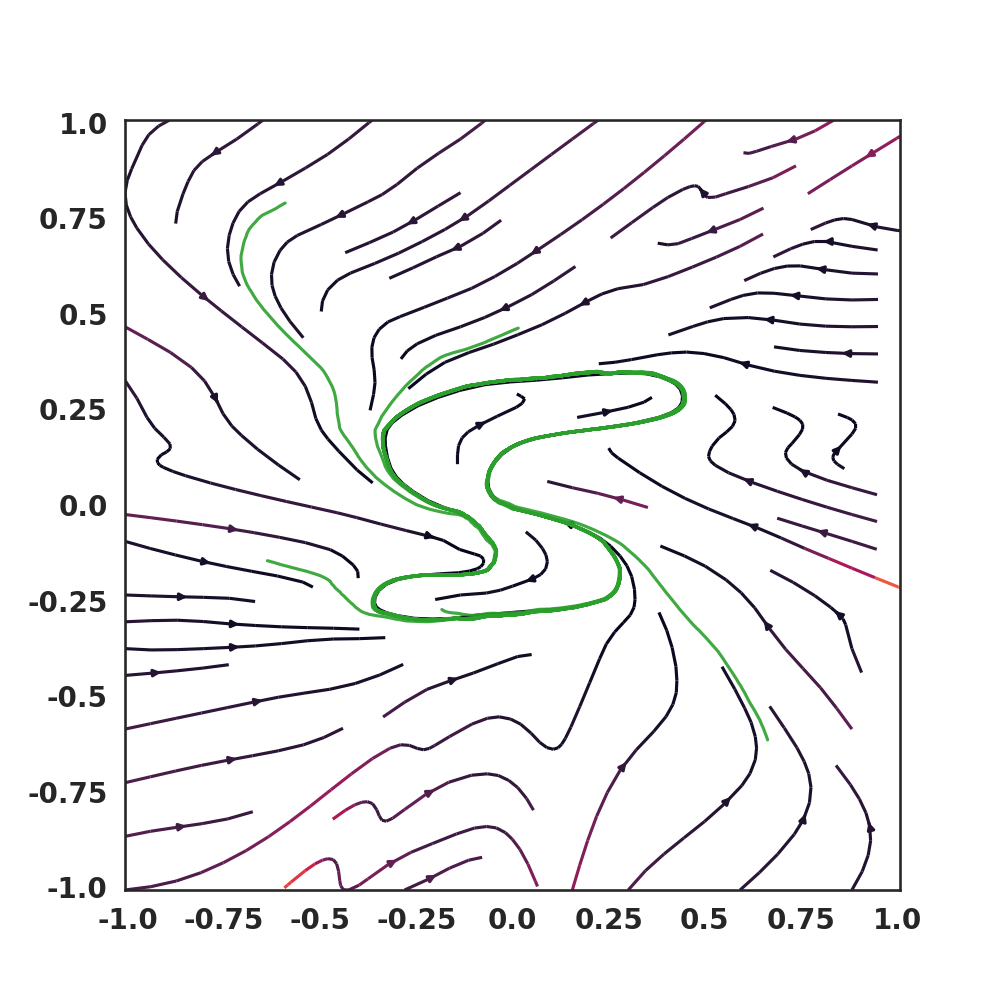}
    \caption{Examples of learned limit cycles. The color in the vector fields represent the magnitude of the velocity, the brighter the faster. In green, trajectories generated by Imitation Flows.}
    \label{fig:iros}
\end{figure*}

\subsection{Classification for IROS letters}
Normalizing Flows compute the exact distribution of the data. This property distinguishes them from other density approximators~\cite{kingma2013auto} that lack the normalization term. Due to this fact, we can compare the obtained probabilities and use them for classification. 

We show this by applying the previously learned letter models for the limit cycles as classification modules.
A simple classifier is built by computing the probability of the given trajectory and  selecting the class with the highest probability
\begin{align}
    k^* = \arg \max_{k \in \mathcal{K}} p(\tau|k).
\end{align}
We have recorded ten new handwritten trajectories for each class, and we use them for testing.
As shown in Fig.~\ref{fig:class}, the classifier gets a high rate of correct classification for all the shapes, with a small prediction error for R and S. The classifier is not considering any scaling or rotating term, which could improve our results. Failure mostly occurs when there is a scale or velocity discrepancy between training and testing trajectories.

\begin{figure*}[t]
\centering
    \begin{minipage}{.45\textwidth}
        \centering
        \includegraphics[height=.8\textwidth]{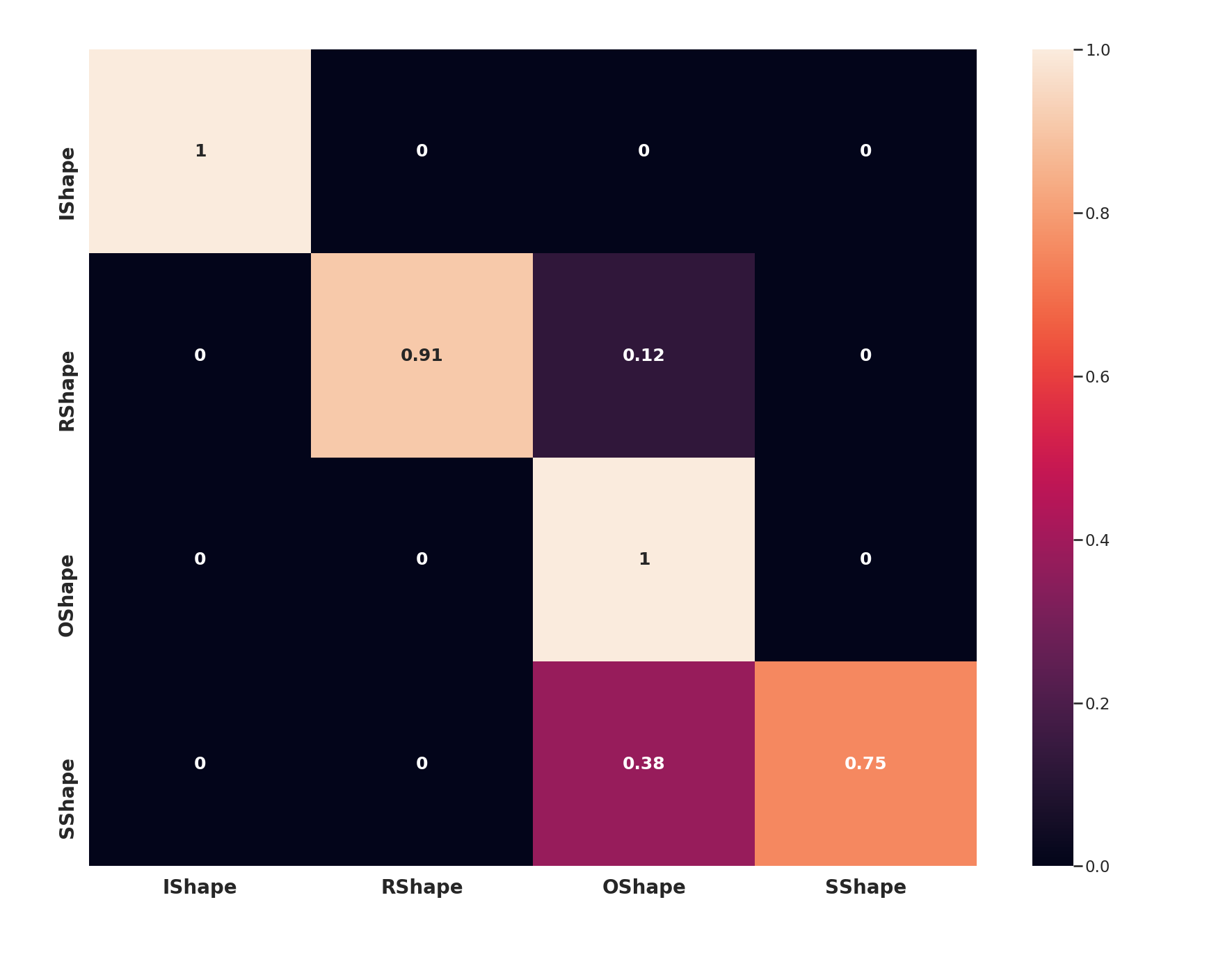}
        \caption{Confussion Matrix for I,R,O,S letters. The values in each cell, represent the probability of a test trajectory(vertical) to be classified as class(horizontal)}
        \label{fig:class}
    \end{minipage}
    \hfill
    \begin{minipage}{.45\textwidth}
        \centering
        \includegraphics[height=.8\textwidth]{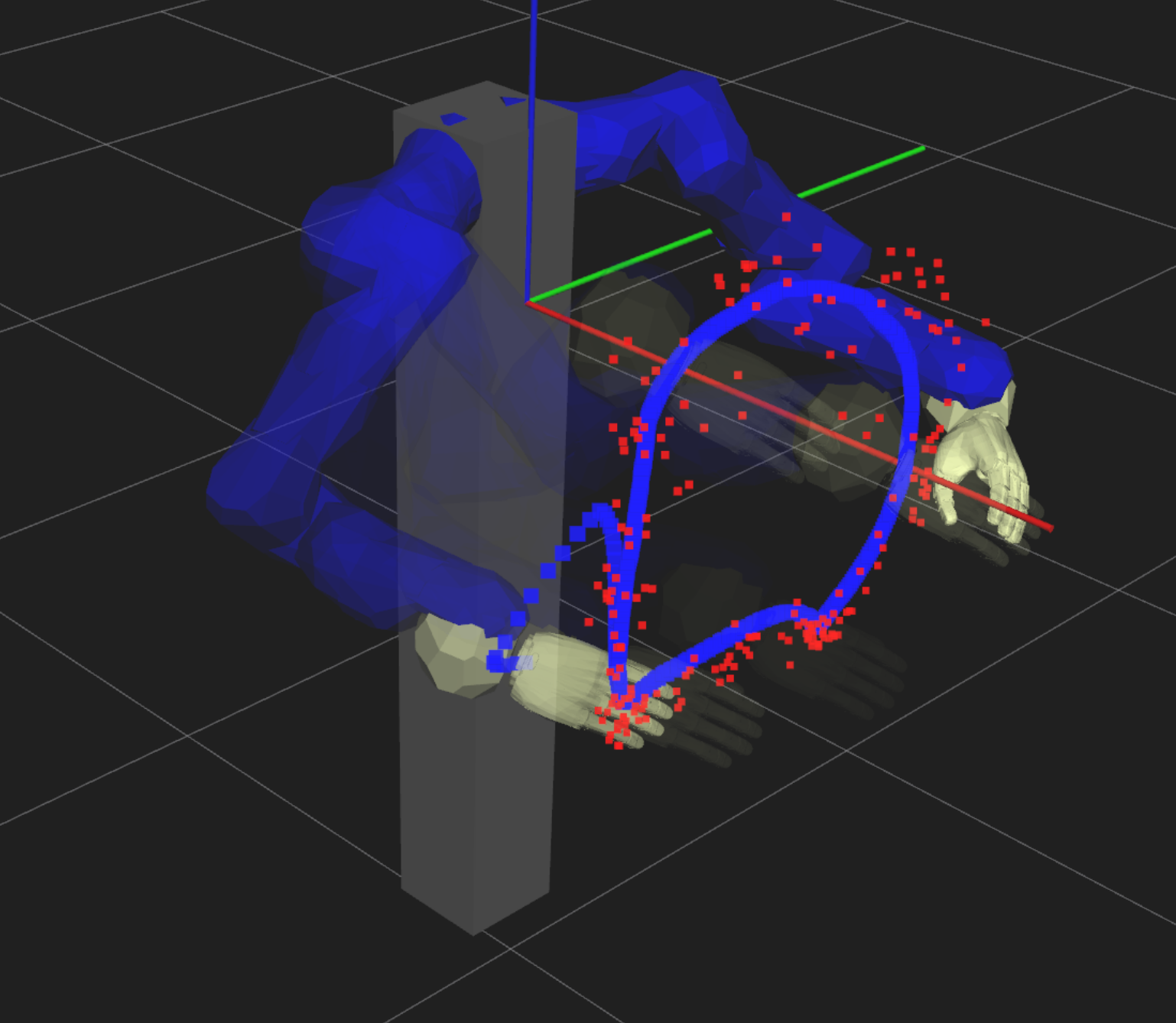}
        \caption{In Blue, the generated trajectory by \textit{ImitationFlows}. In red, samples from given demonstrations.}
        \label{fig:3d}
    \end{minipage}
\end{figure*}

\subsection{Learning drums playing on a robot}\label{sec:robot}

We study the applicability of \textit{ImitationFlow} in a real robotics application, with higher dimensionality(Cartesian Position with fixed orientation). We have recorded trajectories while doing kinesthetic teaching with a DLR-KUKA lightweight robot in drums playing task, presented in Fig.~\ref{fig:real_task}.

To deal with a higher dimensional problem, we switch the Coupling layer~\cite{dinh2016density} with a \gls{maf}~\cite{papamakarios2017masked}. The orthogonal transformation remains on each output of the \gls{maf} layer.
We initialize the model in a similar way to the previous section. We apply a Fourier transform in the first dimension of \gls{pca} space and the main frequency is set as the initial angular velocity. 

We show in Fig.~\ref{fig:3d}, a generated trajectory compared with the demonstrated data in cartesian space. We consider a starting configuration that is far from the limit cycle and generate a full trajectory offline. We show that the model can generate a stable trajectory towards the attractor.

\section{RELATED WORK}\label{RW}

Learning (globally) stable dynamical systems is a crucial topic in the \gls{lfd} community, as these kinds of systems provide motion generation controllers that are robust to perturbations and can generalize the motion in regions of the state space where no demonstration is available.
Previous works on this topic can be divided into two different categories.

The first set of approaches is based on Lyapunov stability analysis.
The common idea behind these methods is to find a candidate function from a family of Lyapunov functions.
One of the first examples of these approaches is the \gls{seds} algorithm~
\cite{khansari2011learning}, that can learn globally stable dynamics towards a goal position. The candidate function is selected from the family of quadratic Lyapunov functions, while the dynamical system is learned using a finite mixture of Gaussian functions. The optimal solution is obtained by maximizing the log-likelihood of the demonstration under the stability constraint using \gls{sqp}.
The main limitation of \gls{seds} is that the dynamics are restricted to contractive ones due to the limitation imposed by the quadratic Lyapunov function family. To overcome this issue, and improve the class of representable movements, the \gls{clfdm} approach~\cite{khansari2014learning} has been proposed. To do so, the authors propose a richer class of Lyapunov functions, the \gls{wsaqf}. The method uses a \gls{clf} to find stable dynamics. Dynamics are modeled using an unstable dynamical system that is stabilized using an additional control signal. 
While \gls{clfdm} is more expressive than the \gls{seds} algorithm, the learning process can be problematic~\cite{neumann2015learning}. 
An alternative solution to the lack of representation power of \gls{seds} is the $\tau-$\gls{seds} algorithm~\cite{neumann2015learning}.
Differently from the \gls{clfdm} approach, this algorithm adds a diffeomorphic transformation on top of the stable dynamics learned by \gls{seds}. However, an application-dependent diffeomorphic transformation must be provided by the designer. Similar in spirit to our work is \cite{perrin2016fast}, which also used a diffeomorphic transformation to inherit the stability properties. While in their case, a kernel-based local translation is used, in our case, we use invertible networks as bijective transformation. In \cite{umlauft2017learning}, the learned stable dynamics were also stochastic. Anyway, the distribution family of the model was imposed and not learned from the data.
The Lyapunov stability principle has been studied also in combination with neural networks.
In~\cite{neumann2013neural}, the authors learn stable dynamics and the Lyapunov candidate from data using a neural network. While they can find accurate results, stability is not guaranteed globally. 

More recently in \cite{Kolter2019Learning} also deep models have been proposed for learning stable dynamic systems. In their paper, a deep neural network is proposed that learns the Lyapunov candidate and the system dynamics jointly by using an \gls{icnn}.

Recently, another set of approaches has been developed, building on the contraction analysis to find a globally stable solution.
In \cite{ravichandar2017learning}, the authors propose a method that uses the same dynamic model as \gls{seds}, but they exploit contraction analysis, instead of Lyapunov functions, to enforce global asymptotical stability.
In \cite{sindhwani2018learning}, a contractive vector field is learned by \gls{rkhs}. This approach can learn a set of equilibrium points and constraint the local curvature using convex optimization. 

There have been a few attempts of extending the Normalizing flows to time series~\cite{Kumar2020VideoFlow, razaghifiltering}. While these approaches benefit from the flexibility of normalizing flows and present remarkable results, they do not have any theoretical guarantees on the stability of the learned dynamical system.

To the best of our knowledge, our method is the only approach that has global stability guarantees for any distribution \gls{sde}, while being able to learn more complex attractors, such as limit cycles.

\section{Discussion and Future Work}\label{conclusions}
In this paper, we presented a Deep Generative Model that extends the Normalizing Flows for learning stable \gls{sde}s. We have demonstrated the Lyapunov stability for a class of \gls{sde}s and propose a learning algorithm to learn them from a set of demonstrated trajectories. 

We show that our model outperforms state-of-the-art algorithms for stable dynamics learning. Moreover, we have shown that our model is not limited to point-to-point dynamics, and we show how it can be extended to complex limit cycle behaviors by plugging a different latent dynamic. The density estimation property of our model has been applied in a classification task, and we show that our model can be extended to higher dimensions to solve a real robot task.

In future works, we will include conditioning to our model, to adapt the dynamics to different environments. Furthermore, we will improve the generalization capabilities of our model in areas without demonstration, to obtain smoother trajectories towards the stable attractor.





\bibliographystyle{ieeetr}
\bibliography{bibliography}

\end{document}